\documentclass{article}
\usepackage{spconf,amsmath,epsfig}
\usepackage{amsfonts}
\usepackage[linesnumbered,ruled]{algorithm2e}
\usepackage{bm}
\usepackage{pdfpages}

\let\OLDthebibliography\thebibliography
\renewcommand\thebibliography[1]{
  \OLDthebibliography{#1}
  \setlength{\parskip}{0pt}
  \setlength{\itemsep}{0pt plus 0.3ex}
}

\usepackage{amsmath}
\usepackage{amsthm}
\newtheorem{theorem}{Theorem}[section]
\newtheorem{lemma}[theorem]{Lemma}
\usepackage{booktabs}
\begin{document}\sloppy

\def\x{{\mathbf x}}
\def\L{{\cal L}}

\title{Noise Dimension of GAN: An Image Compression Perspective }
%
\name{Ziran Zhu$^{*1,2}$, Tongda Xu$^{*2}$, Ling Li$^{\dagger 1}$, Yan Wang$^{\dagger 2}$\thanks{$^*$Ziran Zhu and Tongda Xu contribute equally.$^\dagger$ To whom the correspondence should be addressed. This work is partially supported by the NSF of China (under Grant 92364202), National Science and Technology Major Project (2022ZD0115502) and Baidu Inc. through Apollo-AIR Joint Research Center.}}
\address{$^1$Institute of Software, Chinese Academy of Sciences \\ $^2$Institute for AI Industry Research (AIR), Tsinghua University}

\maketitle

\begin{abstract}
Generative adversial network (GAN) is a type of generative model that maps a high-dimensional noise to samples in target distribution. However, the dimension of noise required in GAN is not well understood. Previous approaches view GAN as a mapping from a continuous distribution to another continous distribution. In this paper, we propose to view GAN as a discrete sampler instead. From this perspective, we build a connection between the minimum noise required and the bits to losslessly compress the images. Furthermore, to understand the behaviour of GAN when noise dimension is limited, we propose divergence-entropy trade-off. This trade-off depicts the best divergence we can achieve when noise is limited. And as rate distortion trade-off, it can be numerically solved when source distribution is known. Finally, we verifies our theory with experiments on image generation.
\end{abstract}
\section{Introduction }
\label{sec:intro}

Generative adversial network (GAN)\cite{Goodfellow2014GenerativeAN} is a type of generative models that maps a noise to target distribution. Previous works on the noise of GAN concludes that the noise helps to stabilize the training procedure \cite{Arjovsky2017TowardsPM}, generate high quality images \cite{Karras2018ASG} and increase the rank of neural network \cite{Feng2020UnderstandingNI}. Besides, the noise of GAN plays an important role for the samples quality \cite{Manisha2020EffectOI}. Most of those analysis treat GAN as a continuous mapping from natural image manifold to continuous natural image signals. This is reasonable as most GANs are trained with continuous Gaussian or Uniform noise. However, those works fail to answer a simple but fundamental question: what is the minimal noise dimension required for GAN?

If we view GAN as a mapping from a continuous distribution to another, this question is unreasonable. From the perspective of arithmetic coding \cite{Witten1987ArithmeticCF}, any data can be compressed into a single number with infinite precision. That is to say, a PNG codec with arithmetic coding is a GAN of noise dimension 1. However in practice, as shown in \cite{Manisha2020EffectOI} and our experiments, insufficient noise dimension does harm the sample quality.

In this paper, we propose to treat GAN as a discrete sampler instead of continuous mapping. From this perspective, we build a connection between the noise dimension required in GAN, and the bitrate to losslessly compress the data it models. More specifically, we prove that for floating point 32 GAN, the dimension of noise required for a GAN is at least $\mathcal{L} / 26.55$, where $\mathcal{L}$ is the minimal bits require to compress the image. To better understand the behaviour of GAN when noise is insufficient, we propose divergence-entropy function. In analogous to rate-distortion function, this function depicts the trade-off between how well the GAN fits and how many noise it requires. When the source distribution is available, we show that this function can be numerically solved by convex-concave programming. We provide a detailed example of divergence-entropy trade-off for a known source distribution. Moreover, we empirically show the existence of this trade-off by experimental results on image generation.

The contributions of our work can be summarized as:
\begin{itemize}
    \item 
    We propose to view GAN as a discrete sampler. From this perspective, we build a connection between the noise dimension in GAN and the bitrate to losslessly compress the source distribution it models.
    \item 
    To depict the behaviour of GAN with insufficient noise, we propose the divergence-entropy trade-off. We show that the divergence-entropy trade-off can be numerically solved when the source distribution is known, and we give an example of it.
    \item 
    We verify the divergence-entropy trade-off by experimental results on image generation for CIFAR10 and LSUN-Church dataset~\cite{Yu2015LSUNCO}, with BIGGAN~\cite{Brock2018LargeSG} and StyleGAN2-ADA~\cite{Karras2020TrainingGA} baseline.    
\end{itemize}

\section{Preliminary: Generative Adversial Network }
The generative adversial network is a type of generative model. Denote the source as $X\sim p_X$, the target of GAN is to learn a parametric distribution $\hat{X}\sim q_{\hat{X};\theta}$ to minimize the divergence with the true distribution. The distribution $q_{\hat{X};\theta}$ is a transform $g_{\theta}(Z)$ of random variable $Z$, which is usually sampled from unit Gaussian distribution.
\label{sec:format}
\begin{gather}
    \min_{\theta} d(p_X, q_{\hat{X};\theta}),\notag \\
    \textrm{where } \hat{X} = g_{\theta}(Z), Z\sim \mathcal{N}(0,I).\label{eq:gan}
\end{gather}

Many divergence $d(.,.)$ has the variational form as Eq.~\ref{eq:gan2} (including Jenson-Shannon divergence, Wassertein distance and any $f$-divergence), where $f(.)$ is a function, $h_1(.)$ and $h_2(.)$ are convex transforms \cite{polyanskiy2014lecture,mescheder2017numerics}.
\begin{gather}
    d(p_X,q_{\hat{X}}) = \min_{\theta} \max_{f\in\mathcal{F}} \mathbb{E}_{ q_{\hat{X}}}[h_1(f(\hat{X}))] - \mathbb{E}_{p_X}[h_2(f(X))]. \label{eq:gan2}
\end{gather}

And therefore, the optimization problem in Eq.~\ref{eq:gan} can be achieved by training a discriminator $f(.)$ maximizing Eq.~\ref{eq:gan2}, with generator $g_{\theta}(.)$ minimizing Eq.~\ref{eq:gan2}. 

\section{Related Works }
There are a couple of works on the noise of GAN from different perspectives. \cite{Bailey2018SizeNoiseTI} shows that a higher dimension of noise requires a larger network. \cite{Feng2020UnderstandingNI} analyzes the role of noise by Riemannian geometry. \cite{Manisha2020EffectOI} finds that the input noise dimension of GAN has significant impact on the sample's quality. However, this work is empirical without theoretical analysis. To the best of our knowledge, we are the first to view GAN as a discrete sampler and derive a bound on the noise dimension it requires.

\section{Noise Dimension Required in GAN }
In the majority of previous works on GANs, the noise $Z$ is viewed as continuous without precision limitation. In that case, the dimension of $Z$ is not important. As for infinite precision continuous $Z$, one dimension is enough to model any data. To better understand this, we can think the decoder of a lossless codec (e.g. PNG) as a GAN, and the bitstream as the binary representation of $Z$. As long as $Z$'s precision is not limited, the bitstream can be arbitrarily long. However, this is not true with real-life computers. As we show in this section, as $Z$'s precision is limited, its dimension does matter. And we propose to understand this by viewing GAN as a discrete sampler.

\subsection{Entropy of IEEE 754 Single-Precision Floating Point Gaussian Distribution }
Prior to discussing the GAN as a discrete sampler, we need to first understand the distribution of noise $Z$ as discrete random variable. The majority of GAN defines noise $Z$ as fully factorized Gaussian distribution. As Gaussian distribution is continuous, it does not have entropy (it only has differential entropy). However, in computers, the sample of continuous distribution is represented 
with 32-bit single-precision floating point format (float32).
A float32 is composed of 1 sign bit, 8 power bits and 23 fraction bits. For example, a single dimension sample $z_i$ with value $1.5$ can be represented as
\begin{align}
    z_i = \underbrace{0}_{+1}\underbrace{01111111}_{\times 2^0}\underbrace{00000000000000000000111}_{\times (1 + 1/2)} = 1.5. \notag
\end{align}
To compute the discrete probability of $z_i$, we need to consider the previous float point number and next float point number of it. More specifically, we have
\begin{align}    
    z_{i-1} &= 00111111101111111111111111111111 \notag \\ 
    &= 1.4999998807907104, \\
    z_{i+1} &= 00111111110000000000000000000001 \notag\\ 
    &= 1.5000001192092896.
\end{align}

With $z_i,z_{i-1},z_{i+1}$, we can easily compute the discrete probability mass function of $z_i$ as the integral
\begin{align}
    p(z_i) = \int_{z_i - (z_i - z_{i-1})/2}^{z_i + (z_{i+1} - z_{i}) / 2} \mathcal{N}(0,1) d x  = 1.4178\times 10^{-8}.
\end{align}
By now, we have defined the discrete distribution of single dimension noise $z$. And we can compute the entropy of $z$ by Monte Carlo as
\begin{align}
    H(Z) &\approx \frac{1}{K}\sum_{i=1}^K -\log p(z_i) = 26.55 \textrm{ bits},
\end{align}
which is approximated using $K=10^7$ samples. This means that a 32 bits floating point with single Gaussian distribution can be losslessly compressed into approximately 26.55 bits. As GAN use diagonal Gaussian noise, for multi-dimension $Z$, we can simply multiply the above result with dimension $n$:
\begin{gather}
    H(Z^n) \approx 26.55n \textrm{ bits.} \label{eq:nh}
\end{gather}

\subsection{Entropy of Other-Precision Floating Point }
Similar to 32 bits floating points, the entropy of other-precision floating point can also be estimated and the result is shown in Table.~\ref{tab:ent}. We assume that 32 bits floating point analysis is used in later analysis unless we emphasis otherwise. Note that we only need to replace the $26.55$ constant by the entropy in Table.~\ref{tab:ent} to obtain the results of other precision floating point.
\begin{table}[htb]
\centering
\begin{tabular}{@{}ll@{}}
\toprule
                  & Entropy (bits) \\ \midrule
Floating Point 16 & 11.36                   \\
Floating Point 32 & 26.55                   \\
Floating Point 64 & 55.56                   \\ \bottomrule
\end{tabular}
\caption{The entropy of other-precision floating point.} 
\label{tab:ent}
\end{table}
\subsection{Noise Dimension Required in Perfect GAN }
As we have understood the noise $Z$ in floating point 32 implementation, we proceed to the dimension of noise required to train perfect GAN, or to say, $d(p_X,q_{\hat{X};\theta}) = 0$. We need the following lemma to connect the entropy of $Z$ to the entropy of generated samples $\hat{X}$:
\begin{lemma} (Transform reduces entropy)
\cite{Cover1991ElementsOI} Assume $g_{\theta}(.)$ is a deterministic transform, we have
\label{lemma:01}
\begin{gather}
    H(Z) \ge H(g_{\theta}(Z)) = H(\hat{X}),
\end{gather}
with equality holds iff $g_{\theta}(Z)$ is bijective.
\end{lemma}
With this lemma, we show that the noise dimension required for perfect GAN is closely connected to the bitrate to losslessly compress the source data:
\begin{theorem} \label{theorem:01} (Noise dimension required for perfect GAN) To achieve perfect divergence $d(p_X,q_{\hat{x};\theta}) = 0$, the minimal dimension of noise $Z$ required is
\begin{gather}
   n \ge \frac{H(X)}{26.55} \ge \frac{\mathbb{E}[\mathcal{L}(X)] - 1}{26.55},
\end{gather}
where $\mathbb{E}[\mathcal{L}(X)]$ is the minimal expected bits required to losslessly encode the source $X$. Further, when $g_{\theta}(.)$ is bijective, 
\begin{gather}
    n \le \frac{H(X)+2}{26.55} \le \frac{\mathbb{E}[\mathcal{L}(X)]+2}{26.55}.
\end{gather}
\end{theorem}
\begin{proof}
    The perfect divergence $d(p_X,q_{\hat{x};\theta}) = 0$ implies that $H(X) = H(\hat{X})$. Using Lemma.~\ref{lemma:01}, we know that $H(Z) \ge H(\hat{X}) = H(X)$. From Eq.~\ref{eq:nh}, we know that $26.55n\ge H(X)$. From Kraft inequality \cite{Cover1991ElementsOI}, we have $H(X)\ge\mathbb{E}[\mathcal{L}(X)] - 1$. When $g_{\theta}(.)$ is bijective, using Lemma.~\ref{lemma:01}, we know that $H(Z) = H(\hat{X})$. Then by Yao's theorem on random number generator \cite{Knuth1976TheCO}, we know that there exist a sequence of binary random variable with expected length $\le H(X)+2$ that generates $X$. Then we can simply use the lossless code of $Z$ as this binary random variable sequence to generate $X$. And the length of the code is related to the entropy by Kraft inequality. 
\end{proof}
In other words, the noise dimension required is at least the minimal bitrate to losslessly compress the source data divide by $26.55$. In practice, the ideal lossless compressor that achieves minimal expected bits $\mathbb{E}[\mathcal{L}(X)]$ might not exist. Therefore, we use the rate of practical lossless coders (e.g. PNG) as an approximation to $\mathbb{E}[\mathcal{L}(X)]$. 
\subsection{Noise Dimension Required in Non-Perfect GAN }
In previous section, we have discussed the noise dimension required for perfect GAN. In this section, we generalize this result to non-perfect GAN. In other words, we study the best divergence that can be achieved when the entropy is limited.

\textbf{Divergence-Entropy Trade-off:} To achieve this target, we first propose divergence-entropy function as follows
\begin{gather}
    d(\epsilon) = \min_{q_{\hat{X}}} d(p_X, q_{\hat{X}})\textrm{, s.t. } H(\hat{X}) \le \epsilon,\notag \\
    \textrm{where } \epsilon \approx 26.55n
\end{gather}

The $d(\epsilon)$ function is the best divergence can be achieved by any generative model $q_{\hat{X}}$ when the entropy $H(\hat{X})$ is limited. Thus, it describes the best divergence of GAN when noise dimension is limited. For example, consider a DC-GAN with Jensen-Shannon divergence as $d(.,.)$. Then with a given noise dimension $n$, we can compute $\epsilon$, and then the best divergence can be achieved is just $d(\epsilon)$. And as we show later, $d(\epsilon)$ can be numerically solved, when the source $p_X$ is known. 

Finally, the $d(\epsilon)$ function has several obvious properties, we list them here without proof:
\begin{itemize}
    \item $d(\epsilon)$ is monotonously non-increasing in $\epsilon$.
    \item If $\epsilon \ge H(X)$, $d(\epsilon) = 0$. This is the conclusion from Theorem~\ref{theorem:01}.
    \item If $\epsilon = 0$, $q_{\hat{X}}(\hat{X}=\hat{x}) = \left\{ 
    \begin{array}{lc}
        1 & \hat{x} = \arg\max p_X \\
        0 & \hat{x} \neq \arg\max p_X \\
    \end{array}
\right.$

\end{itemize}
\textbf{Numeric Solution with Known $p_X$:}
Similar to rate-distortion function $R(D)$, the divergence-entropy function $d(\epsilon)$ can also be solved numerically when $p_X$ is known. However, the Blahut–Arimoto (BA) algorithm \cite{blahut1972computation} is no longer usable. This is because BA algorithm requires both the target and constraints to be convex. While in our case, when the divergence is convex in second argument (which is true for Jenson-Shannon divergence, Wassertein distance and any $f$-divergence), the target $d(p_X,q_{\hat{X}})$ is convex in $q_{\hat{X}}$. However, the constraint $H(\hat{X})\le\epsilon$ is concave in $q_{\hat{X}}$.

Therefore, we resort to disciplined convex-concave program (DCCP) method, which is designated to solve the nonconvex problems. More specifically, a disciplined convex-concave program has the form:
\begin{gather}
    \min \textrm{or} \max o(y) \notag \quad
    \textrm{s.t. } l_i(y) \sim r_i(y) \textrm{, } i=1,...,m, 
\end{gather}
where $y$ is the variable, $o(.)$ is the convex or convcave objective, $l_i(.),r_i(.)$ are convex functions or concave functions and $\sim$ is one of the relational operator $=,\le,\ge$. Obviously, DCCP is more general than convex optimization. In our case, we can formulate the $d(\epsilon)$ function as follows:
\begin{itemize}
    \item (Variable) $q_X(x), x\in\mathcal{X}$, where $\mathcal{X}$ is alphabet of $X$.
    \item (Objective)
    $\min o{(q_X(x))}=\sum_{x\in\mathcal{X}} p_X(x) \log \frac{p_X(x)}{q_X(x)}$
    \item (Probability Constraints) $\forall x\in\mathcal{X}, 0 \le q_X(x)$, and $\sum_{x\in\mathcal{X}} q_X(x)=1$, which means that
    \begin{gather}
        l_i = 0 \le r_i(q_X(x)) = q_X(x), i=1,...,|\mathcal{X}|,\notag \\
        l_{|\mathcal{X}|+1}({q_X(x)}) = \sum_{x\in\mathcal{X}} q_X(x) = r_{|\mathcal{X}|+1} = 1. \notag
    \end{gather}
    \item (Entropy Constraint)
    \begin{gather}
    l_{|\mathcal{X}|+2} = \epsilon \ge r_{|\mathcal{X}|+2}(q_X(x))=\sum_{x\in\mathcal{X}} -q_X(x) \log q_X(x) \notag 
    \end{gather}
\end{itemize}
In general, when the source $p_X$ is known, $d(\epsilon)$ can be solved numerically as a $|\mathcal{X}|$ dimension DCCP problem with $|\mathcal{X}|+2$ constraints.

\textbf{An Example} We implement the numerical solver of $d(\epsilon)$ in DCCP which is the extension of pyCVX \cite{diamond2016cvxpy}. Here, we provide a toy size example, where the source distribution is a $7$ dimensional categorical distribution. The probability of each class is shown in Fig.~\ref{fig:mni_toy}.(a). The entropy for this distribution is $H(X)= 1.857$. We traverse $\epsilon$ from $0.05$ to $2.00$ and using our solver to solve $d(\epsilon)$ as Fig.~\ref{fig:mni_toy}.(b). As shown in this figure, the $d(\epsilon)$ is monotonously non-increasing, and reaches $0$ as $\epsilon \ge H(X)$ as we expected.
\begin{figure}[thb]
\centering
 \includegraphics[width=\linewidth]{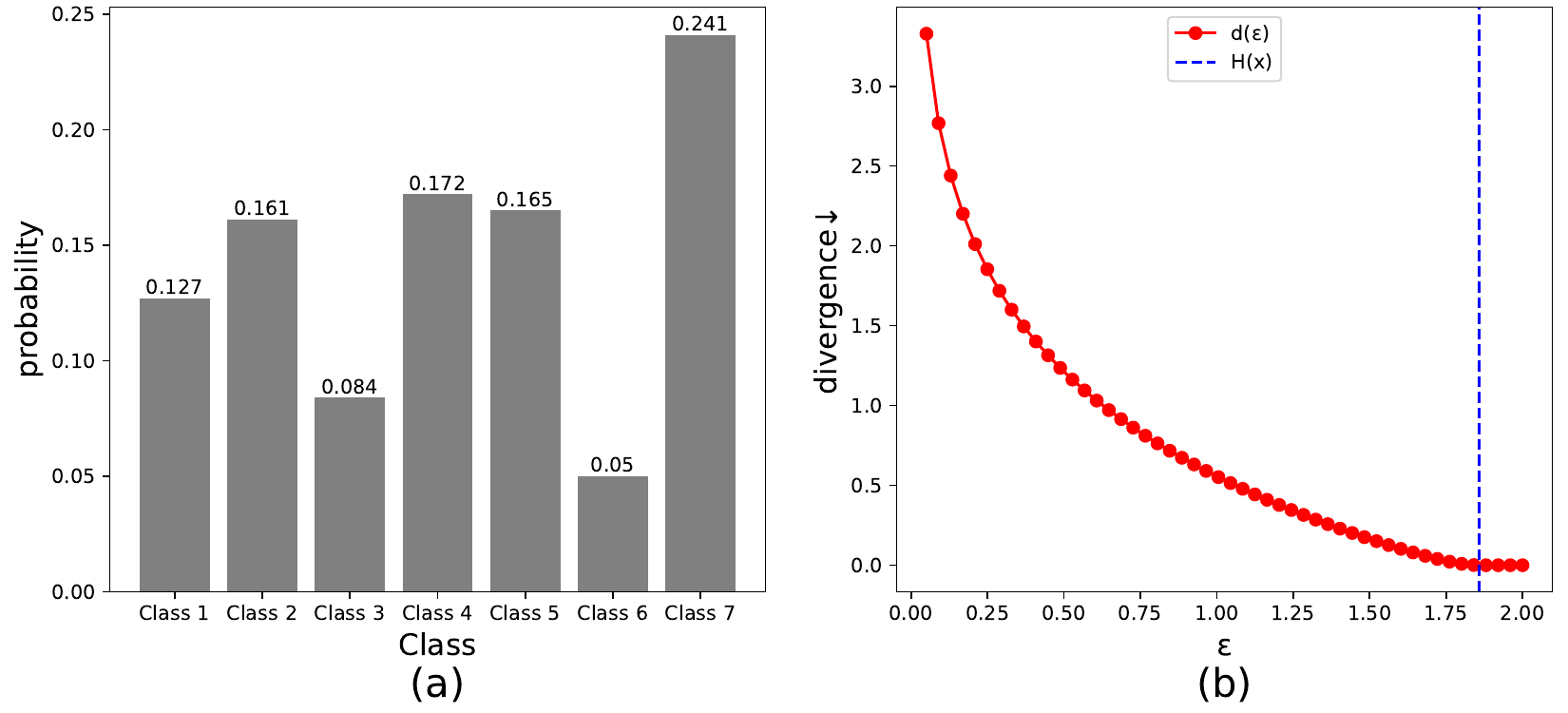}
 \caption{(a) Source distribution. (b) $d(\epsilon)$ curve.}
 \label{fig:mni_toy}
\end{figure}

\section{Experiments}
\begin{figure}[htbp]
\centering
 \includegraphics[width=\linewidth]{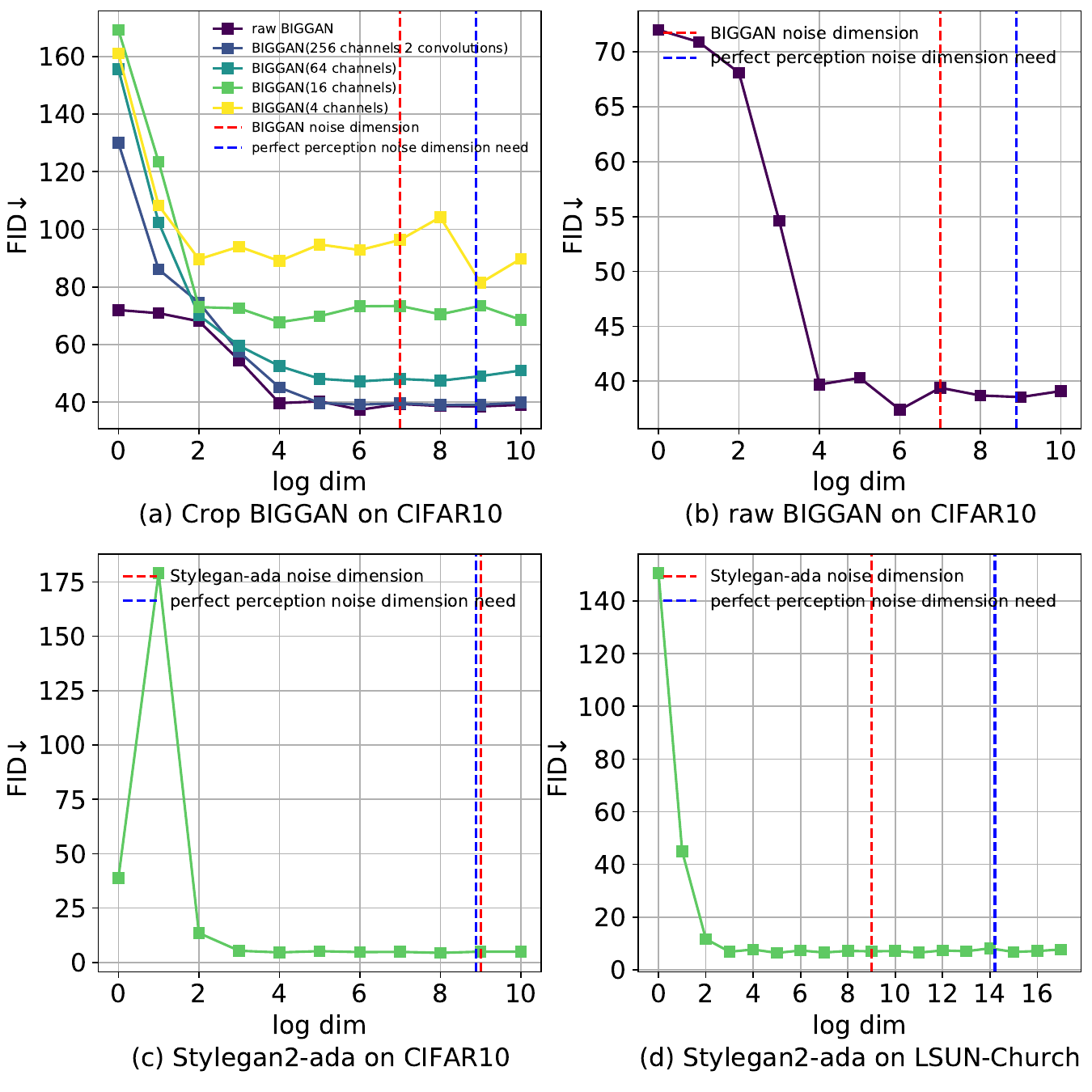}
 \caption{The result of FID on CIFAR10 and LSUN-Church.}
 \label{fig:mni}
\end{figure}

\begin{figure}[htbp]
\centering
 \includegraphics[width=\linewidth]{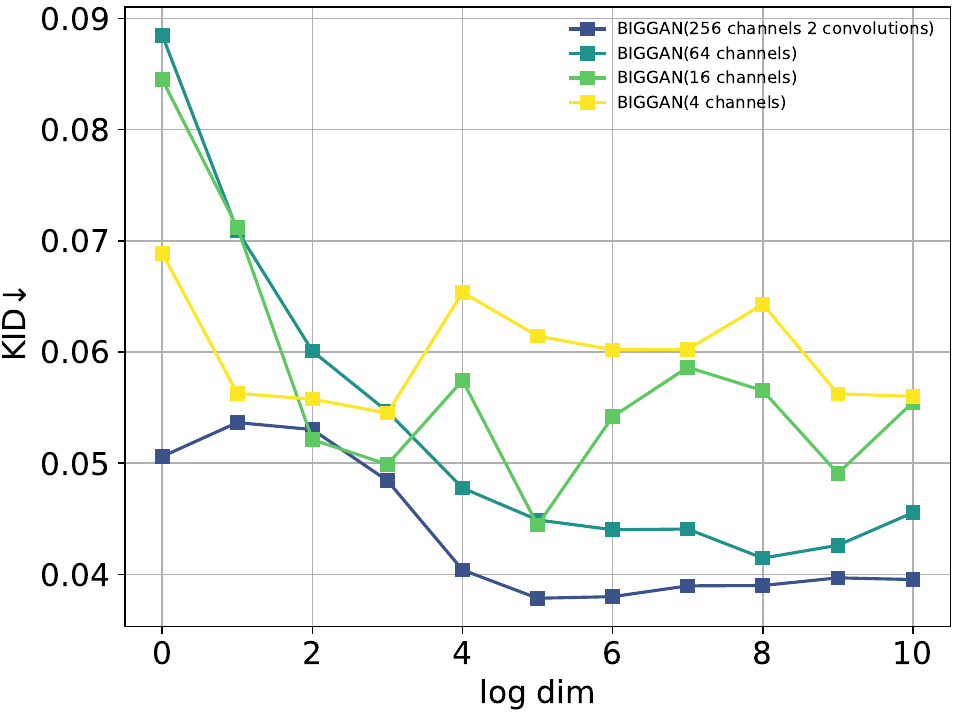}
 \caption{BIGGAN's KID on CIFAR10.}
 \label{fig:mni}
\end{figure}

\begin{figure}[htbp]
\centering
 \includegraphics[width=\linewidth]{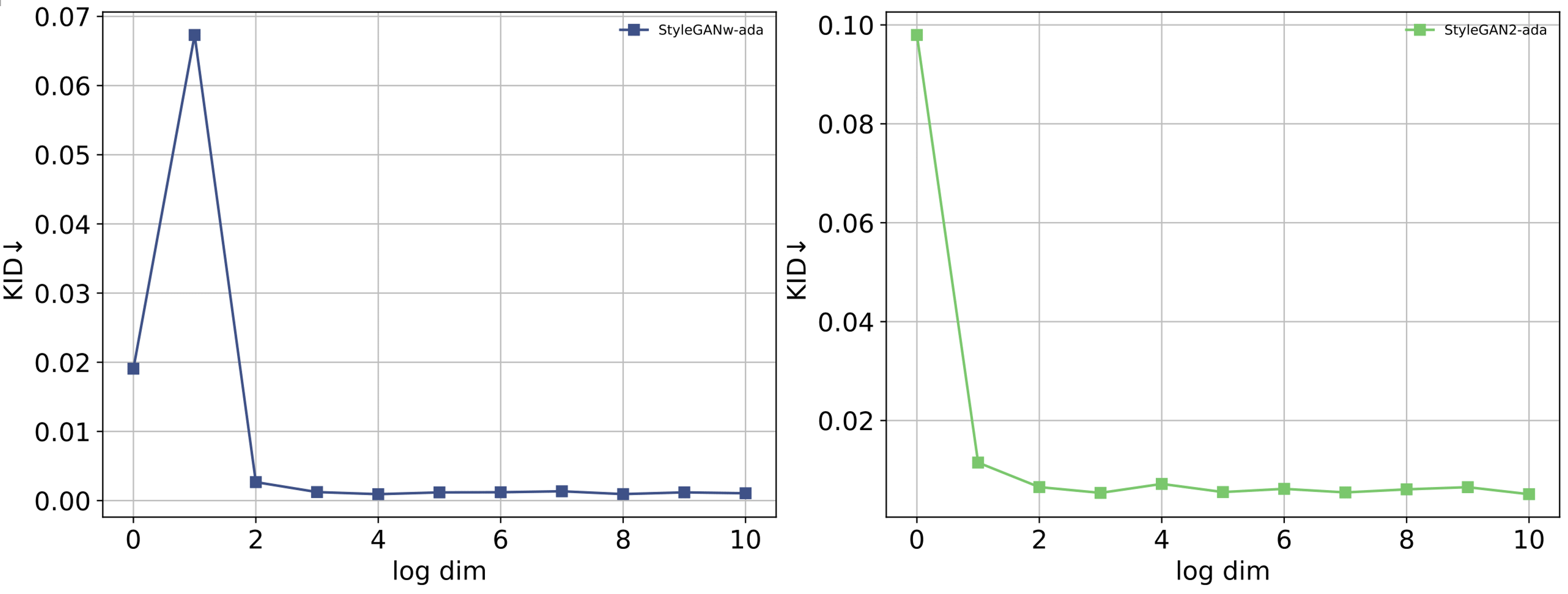}
 \caption{StyleGAN-ada's KID on CIFAR10 and LSUN-Church.}
 \label{fig:mni}
\end{figure}

\label{sec:majhead}
In this section, we empirically verify that the divergence-entropy trade-off exists by studying the behaviour of GAN when noise dimension is limited. NVIDIA GPUs have more optimized support for single-precision (FP32) calculations. Double-precision (FP64) calculations might be slower and hard to use. Thus, we do experiments with Single-precision float points setting.
\subsection{Experiment Setup}
\label{ssec:subhead} 
\textbf{Dataset} We choose CIFAR10~\cite{Krizhevsky2009LearningML} and LSUN-Church~\cite{Yu2015LSUNCO} datasets, which are widely adopted in image generation.

\noindent\textbf{Metrics} We choose Fréchet inception distance (FID) and Kernel inception distance (KID) as the metric to evaluate $d(.,.)$, which is widely adopted in GAN.

\noindent\textbf{Baselines and Training} We use BIGGAN \cite{Brock2018LargeSG} and StyleGAN2-ADA~\cite{Karras2020TrainingGA} as baselines. We train BIGGAN and StyleGAN2-ADA on CIFAR10 and LSUN-Church datasets with training details same with the original papers. We vary the input noise dimension $n$ and evaluate the resulting FID. For BIGGAN, we also attempts to reduce the network capacity, and observe the noise dimension $n$ required to achieve minimal FID.



\begin{table}[t]
\centering
\begin{tabular}{@{}lll@{}}
\toprule
        & CIFAR10    & LSUN-Church  \\ \midrule
PNG     & 2271 Bytes & 103897 Bytes \\
WebP    & 1807 Bytes & 69246 Bytes  \\
JPEG-XL & 1576 Bytes & 62940 Bytes  \\ \bottomrule
\end{tabular}
\caption{Average Bytes that different lossless codec need to compress images} \label{tab:cap}
\end{table}

\begin{table}[t]
\centering
\begin{tabular}{@{}lll@{}}
\toprule
                  & CIFAR10 & LSUN-Church \\ \midrule
Floating Point 16 & 1110    & 44324       \\
Floating Point 32 & 475     & 18966       \\
Floating Point 64 & 227     & 9063        \\ \bottomrule
\end{tabular}
\caption{Noise Dimension Perfect GAN need} 
\end{table}

\subsection{Results and Analysis}
\label{sssec:subsubhead}

\begin{figure}[htbp]
\centering
 \includegraphics[width=\linewidth]{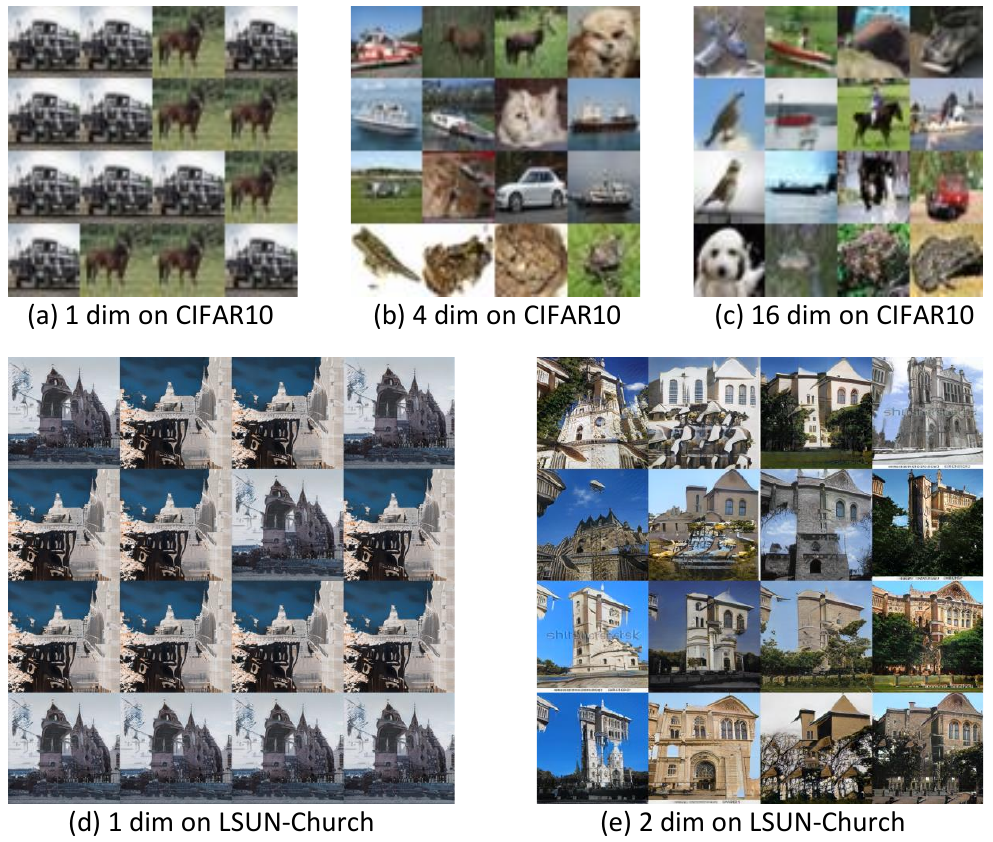}
 \caption{Samples using different noise dim.}
 \label{fig:sam}
\end{figure}

We use practical lossless coders to approximate $\mathbb{E}[\mathcal{L}(X)]$. Table 2 show the average bytes of CIFAR10 and LSUN-Church datasets compressed with different lossless codecs. From Table 2, JPEG XL achieves the best compression ratio, which approximates that the minimal noise dimension for CIFAR10 is 475, for LSUN-Church is 18966.

We evaluate BIGGAN and StyleGAN-ADA with reduced noise dimension, as shown in Fig.~\ref{fig:mni}(b)(c)(d). It is shown that as the noise dimension reduces, the FID goes up. This verifies the existence of divergence-entropy trade-off. However, the noise dimension $n$ with a reasonable low FID is smaller than the minimal noise dimension $n$ we computed. This is because the minimal FID that can be achieved is not $0$, and it is limited by network capacity. To further verify this, we reduce the network capacity of BIGGAN and shows the noise dimension and FID in Fig.2(a). As the network goes smaller, the minimal FID goes up, and the noise dimension with lowest FID also decreases. The KID results in Fig.3 and Fig.4 show similar results. Further, we present the images generated with different noise dimensions in Fig.5. Obviously, GAN with limited noise dimension has limited sample diversity. We find that GAN especially with one dimension noise can only generate a few categories of images.

\section{Conclusion}
In this paper, we propose to view GAN as a discrete sampler. By that, we connect the lowerbound on the noise dimension required for GAN with the bitrate to compress the source data losslessly. We further propose divergence-entropy trade-off, which depicts the best divergence of GAN when noise is limited. We propose a numerical approach to solve the divergence-entropy trade-off when the source distribution is known. Empirically, we verify the existence of this trade-off by the experiment of GAN.

\newpage

\bibliographystyle{IEEEbib}
\bibliography{icme2023template}

\end{document}